\documentclass[conference]{IEEEtran}

\usepackage{amsmath}
\usepackage{amsthm}
\usepackage{newtxmath}
\usepackage{graphicx}
\usepackage{bm}
\usepackage[T1]{fontenc}
\usepackage{comment}
\usepackage[colorlinks=true, urlcolor=blue, citecolor=black, linkcolor=black]{hyperref}
\usepackage{silence}
\usepackage{subcaption}
\usepackage{tikz}
\usetikzlibrary{arrows.meta}
\usepackage{cite}
\newtheorem{theorem}{Theorem}

\IEEEoverridecommandlockouts

\begin{document}

\title{Spectral Higher-Order Neural Networks Have Sharp Expressivity Bounds}

\author{
    \IEEEauthorblockN{Gianluca Peri, Diego Febbe, Duccio Fanelli}
    \vspace{5pt}
    \IEEEauthorblockA{Department of Physics, University of Florence \& INFN, Italy \\ \\
    \texttt{\{gianluca.peri, diego.febbe, duccio.fanelli\}@unifi.it}}
    \thanks{This work has been accepted for presentation at COMPENG 2026 
-- 2026 IEEE Workshop on Complexity in Engineering.}
}

\maketitle
\begin{abstract}
    \noindent
    Neural hypergraphs are a natural generalization of neural networks, the reference models in modern machine learning. Yet, their deployment has proven demanding: the number of weighted hyperedges required leads to an intractable parameter explosion. However, a novel parametrization that leverages spectral attributes for neural hypergraphs has been recently proposed, that enables to recycle parameters via a weight sharing scheme and consequently yields a significant reduction of the associated computational cost. Preliminary tests carried out on spectral higher-order architectures pointed to meaningful improvements in both performance and interpretability. Building on these results, we advance the benchmarking efforts by evaluating the spectral higher order framework on N-bit parity tasks, a well-established testbed known to be particularly challenging. As we will convincingly argue, Spectral Higher-Order Neural Networks (\textsc{shonn}s) possess a versatile and highly tunable hypothesis space.
\end{abstract}

\begin{IEEEkeywords}
    machine learning, neural networks, hypergraphs
\end{IEEEkeywords}

\section{Introduction}

While in the past \textit{good old-fashioned artificial intelligence} (\textsc{gofai}) solved problems by creating \textit{ad-hoc} algorithms, modern \textit{machine learning} (\textsc{ml})  defines instead learning algorithms, that in turn produce solving routines. Such an indirect approach proved incredibly powerful, leading to a boom in modern artificial intelligence capabilities that nowadays spans from simple agentic behavior \cite{PERI2026118356} to language models \cite{Turner2023AnIT}. The foundational paradigm underlying nearly all contemporary machine learning algorithms is the \textit{neural network} (\textsc{nn}), where solutions are generated by tuning the connections' weights of a massive network. 

Over the years, numerous network topologies have been proposed, but the most utilized remains the \textit{feedforward neural network} (\textsc{fnn}): neurons (\textit{i.e.} the network's nodes) are arranged into layers, and then connected by unidirectional bundles of links (see Figure \ref{fig:left}). A natural generalization is represented by \textit{higher-order neural network} (\textsc{honn}) \cite{Rumelhart1986ParallelDP, Giles1987LearningIA, Feng2018HypergraphNN, Ebli2020SimplicialNN}, where hyperlinks are also introduced (see Figure \ref{fig:right} for a pictorial representation of the interaction scheme). Although \textsc{honn}s feature a clearly more general topology, the adoption of even their simplest variant—namely, the triadic \textsc{honn} shown in Figure \ref{fig:right}—has been rendered infeasible due to parameter scaling. Indeed, whereas standard neural networks yield an $O(N^2)$ parameter scaling (with $N$ representing the layer size), a naive triadic \textsc{honn} implementation—associating a single weight with each hyperlink—scales as $O(N^3)$. Such parameter explosion inevitably hinders the training of these higher-order networks. To overcome these limitations different solutions have been devised \cite{Shin1991ThePN, chrysos2021deep, chrysos2022augmenting}.

In addition to the above, a new parameterization for triadic \textsc{honn}s has been recently proposed, drastically reducing the parameter count via a parameter-sharing scheme that leverages the spectral properties of the underlying hypergraph. Higher-order networks based on this latter parameterization scheme benefit from the most compact $O(N^2)$ scaling known in the literature, and are referred to as \textit{spectral} \textsc{honn}s (\textsc{shonn}s) \cite{Peri2026SpectralHN} (see Appendix \ref{app:spectral-parameterization} for mathematical details on the spectral parameterization).

An interesting property of \textsc{shonn}s is that, contrary to standard \textsc{nn}s, they can operate without enforcing ad hoc non-linearities (\textit{e.g.}, \textit{ReLU}). This is made possible by the hyperlinks, which provide a source of non-linearity inherent in the network's topology. As we shall see, this notable property has practical implications: specifically, it allows to fine-tuning the model's expressiveness, up to a level that is simply out of reach for standard \textsc{nn}s.

The rest of the manuscript is structured as follows: in Section \ref{sec:honns-and-shonns} we shall briefly go through the math underpinning triadic \textsc{honn}s, and then pivot to illustrate the details regarding the parameterization of \textsc{shonn}s. From there, in Section \ref{sec:shonns-on-parity}, we shall substantiate the claims regarding \textsc{shonn}s' tunability, and test their performance on $N$-bit parity benchmarks. Finally, in Section \ref{sec:conclusions}, we will sum up and elaborate on future research directions.

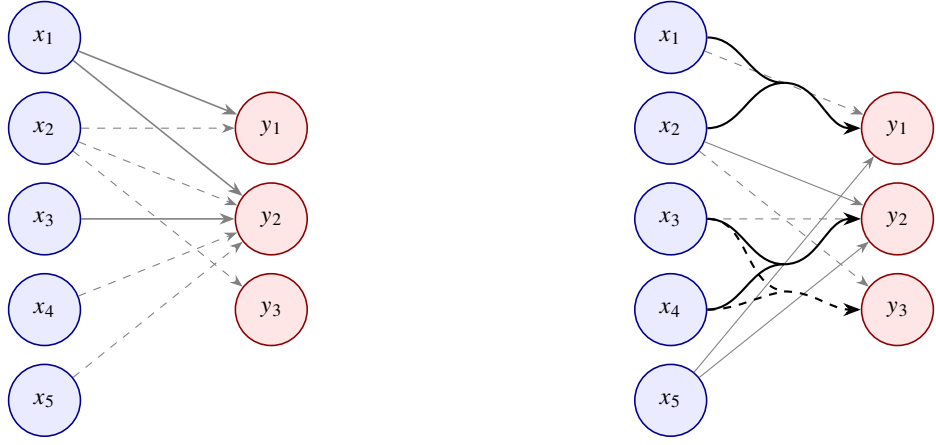
\begin{figure*}[t]
     \centering
     \begin{subfigure}[b]{0.45\textwidth}
         \centering
         \begin{center}
\begin{tikzpicture}[
    >=Stealth,
    neuron/.style={
        circle, semithick, minimum size=0.95cm, inner sep=0pt,
        font=\sffamily\small
    },
    input/.style ={neuron, draw=blue!55!black,   fill=blue!8},
    output/.style={neuron, draw=red!55!black,    fill=red!10},
    link/.style     ={->, semithick},
    backlink/.style ={->, thin, black!50, dashed}
]
\def\nodeSpacing{1.2}
\def\inX{0}
\def\outX{3}

\foreach \i in {1,...,5}
    \node[input] (x\i) at (\inX, { (3-\i)*\nodeSpacing }) {$x_\i$};

\node[output] (y1) at (\outX,  1*\nodeSpacing) {$y_1$};
\node[output] (y2) at (\outX,  0*\nodeSpacing) {$y_2$};
\node[output] (y3) at (\outX, -1*\nodeSpacing) {$y_3$};

\draw[link, black!50]    (x1) -- (y1);
\draw[link, black!50]      (x1) -- (y2);
\draw[link, black!50]    (x3) -- (y2);

\draw[backlink] (x2) -- (y1);
\draw[backlink] (x2) -- (y2);
\draw[backlink] (x2) -- (y3);
\draw[backlink] (x4) -- (y2);
\draw[backlink] (x5) -- (y2);

\end{tikzpicture}
\end{center}
         \caption{Standard connectivity of a neural network.}
         \label{fig:left}
     \end{subfigure}
     \begin{subfigure}[b]{0.45\textwidth}
         \centering
         \begin{center}
\begin{tikzpicture}[
    >=Stealth,
    neuron/.style={
        circle, semithick, minimum size=0.95cm, inner sep=0pt,
        font=\sffamily\small
    },
    input/.style ={neuron, draw=blue!55!black, fill=blue!8},
    output/.style={neuron, draw=red!55!black,  fill=red!10},
    backlink/.style={->, thin, black!50, dashed},
    link/.style={->, thin, black!50},
    hoe/.style={thick}   
]
\def\nodeSpacing{1.2}
\def\inX{0}
\def\outX{3}
\def\midX{1.5}

\foreach \i in {1,...,5}
    \node[input] (x\i) at (\inX, { (3-\i)*\nodeSpacing }) {$x_\i$};

\node[output] (y1) at (\outX,  1*\nodeSpacing) {$y_1$};
\node[output] (y2) at (\outX,  0*\nodeSpacing) {$y_2$};
\node[output] (y3) at (\outX, -1*\nodeSpacing) {$y_3$};

\draw[backlink] (x1) -- (y1);
\draw[link] (x2) -- (y2);
\draw[backlink] (x3) -- (y2);
\draw[backlink] (x2) -- (y3);
\draw[link] (x5) -- (y2);
\draw[link] (x5) -- (y1);

\coordinate (M1) at (\midX,  1.5*\nodeSpacing);
\draw[hoe, black] (x1) to[out=0, in=180] (M1);
\draw[hoe, black] (x2) to[out=0, in=180] (M1);
\draw[hoe, black, ->] (M1) to[out=0, in=180] (y1);

\coordinate (M2) at (\midX, -0.5*\nodeSpacing);
\draw[hoe, black] (x3) to[out=0, in=180] (M2);
\draw[hoe, black] (x4) to[out=0, in=180] (M2);
\draw[hoe, black, ->] (M2) to[out=0, in=180] (y2);

\coordinate (M3) at (\midX, -0.8*\nodeSpacing);
\draw[hoe, black, dashed] (x3) to[out=0, in=170] (M3);
\draw[hoe, black, dashed] (x4) to[out=0, in=190] (M3);
\draw[hoe, black, ->, dashed] (M3) to[out=0, in=180] (y3);

\end{tikzpicture}
\end{center}
         \caption{Connectivity of a standard higher-order (\textit{triadic}) net.}
         \label{fig:right}
     \end{subfigure}
     
     \caption{Schematics of different neural architectures. In a standard neural network (panel \ref{fig:left}) the first-layer neurons (blue) generate output activations (denoted by $y_k$) on the neurons (red) that define the second-layer, via pairwise links (gray arrows), followed by a local non-linearity (not shown). Triadic higher-order networks (panel \ref{fig:right}) extend this architecture by including hyperlinks (in black) that capture three-body interactions (two first-layer neurons associated to one second-layer neuron).}
     \label{fig:forward-passes}
\end{figure*}

\section{HONNs and SHONNs}
\label{sec:honns-and-shonns}

In the following we shall refer exclusively to a single forward propagation between neural layers. The symbol $x$ will refer to  activations on the first-layer, whereas $y$ identify the activity associated to neurons of the second layer; $w$ stands for the involved model's weights. Symbols $\lambda$ or $\phi$ will be used to identify the relevant spectral attributes (namely, eigenvalue or eigenvector elements).

The standard forward propagation of a neural network can be cast as follows: 
\begin{equation}
    y_k = \sum _{i=0}^N w_{ki}x_i.
\end{equation}
At variance, the simplest triadic \textsc{honn} propagates the information according to the following mathematical recipe: 

\begin{equation}
    y_k = \sum _{i=0}^N w_{ki}x_i + \sum _{0 \leq i \leq j}^{N-1} \tilde{w}_{kij}x_ix_j.
    \label{eq:direct-space-triadic}
\end{equation}

Note the presence in the above expression of the \textit{bias neuron} in the linear part of the forward pass \footnote{It is worth remarking that, in principle, a general triadic higher-order network could follow to the more complex transformation:
\[
    y_k = \sum _{i=0}^N w_{ki}x_i + \sum _{0 \leq i \leq j}^{N-1} f(x_i,x_j; \bm{\theta}),
\]
where $f$ represents a generic three-body coupling function, and $\bm{\theta}$ is the associated parameter vector. However, it is customary to consider only multiplicative couplings \cite{Giles1987LearningIA}. This practice amounts to considering only the first order effects of a possible, more complex, coupling function $f$.
}.
We can immediately notice the $O(N^2)$ \textit{vs.} $O(N^3)$ scaling that we anticipated in the previous section. The spectral parameterization solves this problem via a reformulation of the higher-order interactions:
\begin{equation}
    y_k = \sum _{i=0}^N w_{ki}x_i + \sum _{0 \leq i \leq j}^{N-1} (\tilde{\lambda} _{ij}^{(in)}-\tilde{\lambda} _k^{(out)})\tilde{\phi}_{kij}x_ix_j.
    \label{eq:extended-spectral-param}
\end{equation}
Denote by $N_{in}$ the number of neurons in the first layer, and $N_{out}$ the number in the second layer. Then, $\tilde{\lambda} _{ij}^{(in)}$ is the element of a $N_{in} \times N_{in}$ matrix of (spectral) parameters, $\tilde{\lambda} _k^{(out)}$ is one entry of  a vector of length $N_{out}$, and finally $\phi _{kij}$ defines a $N_{out} \times N_{in} \times N_{in}$ parameter tensor. The rigorous underlying formalism is illustrated in details in  Appendix \ref{app:spectral-parameterization}. The above decomposition in terms of generalized eigenvectors and eigenvalues opens the door to a novel training paradigm where only a portion of the spectral parameters is optimized \cite{chicchi2021training}; for example, one could choose to train only the eigenvalues and a portion of the eigenvectors' tensor, or, as we shall, solely optimize the eigenvalues, yielding a consequent compression of the parameter scaling down to $O(N^2)$, even for the higher-order counterpart of the neural model. \textit{A priori}, While training only a subset of the model's parameters risks hindering expressivity and performance, existing literature rules out this risk for this specific case, provided enough neurons are present \cite{Peri2026SpectralHN}.

At this point, it is mandatory to discuss one of the main differences between standard neural networks and higher-order ones: the handling of non-linear mappings. In the standard case the hypothesis space of the model is rendered non-linear via the application of a suitable \textit{activation function} $\sigma$: When a signal reaches one of the network's hidden layers, the activation function $\sigma$ is applied to it neuron-wise. Without this crucial step, the neural network would be confined to a strictly linear hypothesis space. On the contrary, higher-order neural models are intrinsically expressive: the non-linearity is inherently encapsulated into the topology of the model via its higher-order neural couplings. Because this characteristic obviates the need for an explicit non-linearity $\sigma$, higher-order models can be successfully deployed either way. In the following we shall specifically focus on the $\sigma$-less case. Remarkably, it has been proven that $\sigma$-less triadic neural networks—when equipped with a sufficient number of neurons across enough layers—can approximate any continuous function, even if only their eigenvalues are optimized \cite{Peri2026SpectralHN}. In the following we will expand on this notable result, by empirically investigating the expressivity of the triadic models under eigenvalue-only training, for specific parity tasks as hereafter detailed.

\section{SHONNs on parity tasks}
\label{sec:shonns-on-parity}

\begin{figure*}[ht]
    \centering
    \includegraphics[width=0.4\linewidth]{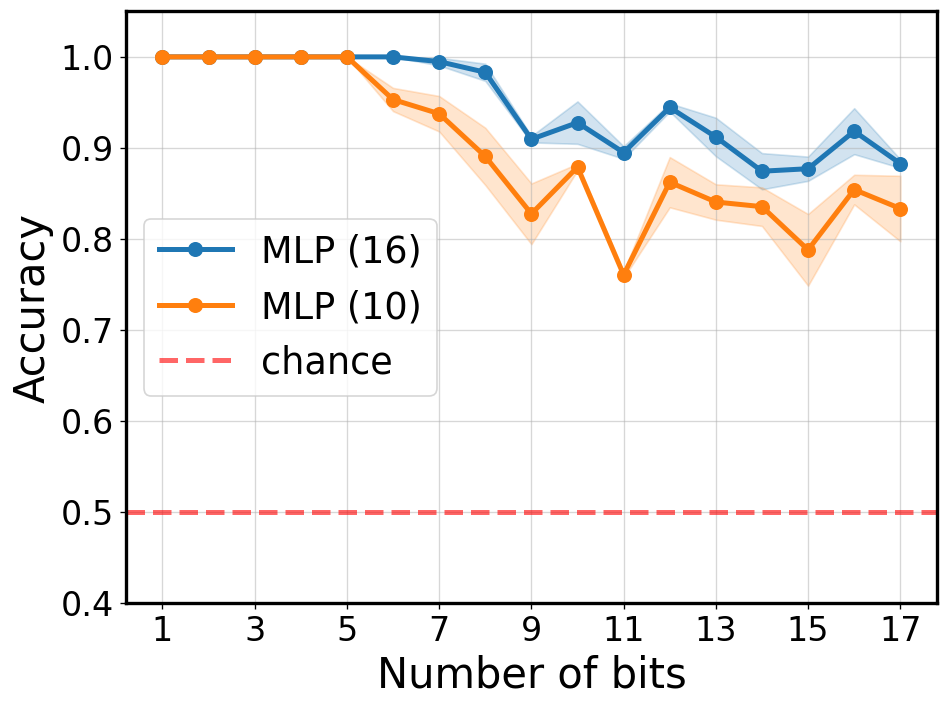}
    \caption{Results of $N$-bit parity benchmarks (up to $17$ bits) on simple $3$-layer \textsc{mlp}s, with $16$ and $10$ hidden units respectively (further implementation details in Appendix \ref{app:details-on-parity-exps}). As proven in Appendix \ref{app:relu-mlp-on-parity}, $MLP(10)$ (the \textsc{mlp} with $10$ hidden units) should hold sufficient expressive power to exactly solve the tasks  up to $10$ bits; however, in practice, the learning process is unable to optimize the network to this theoretical ideal. A similar pattern is reported for $MLP(16)$. Notice that  adding extra neurons helps the optimizer to find a better parameter configuration. In practice, it seems that an $MLP(N)$ is not able to learn how to solve the $N$-bit parity task, despite having the potentiality for it. Adding more neurons helps the optimization process, but also pushes the theoretical expressivity of the model forward, thus creating a fuzzy situation where no sharp expressivity bound can be reached trough learning.}
    \label{fig:mlp}
\end{figure*}

Even for the simplest standard model—a three-layer \textsc{mlp}—the \textit{universal approximation theorem} \cite{Cybenko1989ApproximationBS} guarantees that the network can approximate any continuous map with arbitrary precision, provided a sufficient number of hidden neurons are used. This kind of unbounded expressivity can be extremely useful, but it can also lead to overfitting problems. 

Finite-size, spectral, higher-order networks, when used in combination with neural activation functions, are more expressive than traditional neural networks and can, therefore, suffer from similar overfitting problems. However, the higher-order models analyzed herein can also be effectively utilized without activation functions. In this latter setting, the ensuing models are characterized by a restricted hypothesis space whose span can be tuned via simple architectural choices (\textit{i.e.}, the number of layers). Specifically, it has been shown that when integrated into larger neural architectures, these models can act as implicit regularizers \cite{Peri2026SpectralHN}.

To further explore the expressivity of SHONNs, we here carry out N-bit parity benchmarks on them. Parity tasks are a notoriously challenging problem for standard neural networks: a binary input vector of size $N_I$ has to be classified as \textit{even} or \textit{odd}, depending on its parity \footnote{As an example, the $N_I=3$ input vector $(+1,-1,-1)$ has an even number of $-1$ and is therefore \textit{even}.}. The intrinsic complexity of the parity mapping stems from its sensitivity: even for large inputs, flipping a single activation in the first layer must trigger a completely opposite output pattern. For $N_I = 2$ the associated parity task collapses to the \textit{xor} mapping, while for large $N_I$ the learning problem becomes challenging really fast.

In Figure \ref{fig:mlp} are reported the performances of simple one-hidden-layer \textsc{mlp}s, equipped with \textit{ReLU} activated hidden units, on parity tasks up to $17$ bits (further details on the experimental setup are presented in Appendix \ref{app:details-on-parity-exps}).

It is possible to show (see Appendix \ref{app:relu-mlp-on-parity}) that a \textit{ReLU} \textsc{nn} with $N$ hidden neurons is capable of exactly solving all the parity tasks up to $N$-bits. Beyond theoretical bounds, the optimization process in practice fails to drive the network to the global minimum of the loss function, leading to sub-optimal performance even under ideal learning conditions.

While making the network wider could, in principle, allow the optimizer to find a valid parity mapping for larger $N$, doing so forces an increase in the model's theoretical expressivity, thereby sacrificing the regularization guarantees. \textsc{shonn}s avoid this problem by decoupling the ease of optimization (that scales with the neural width) from the maximum complexity of the expressed mapping (scaling with the network's depth).

\begin{figure*}[ht]
    \centering
    \includegraphics[width=0.7\linewidth]{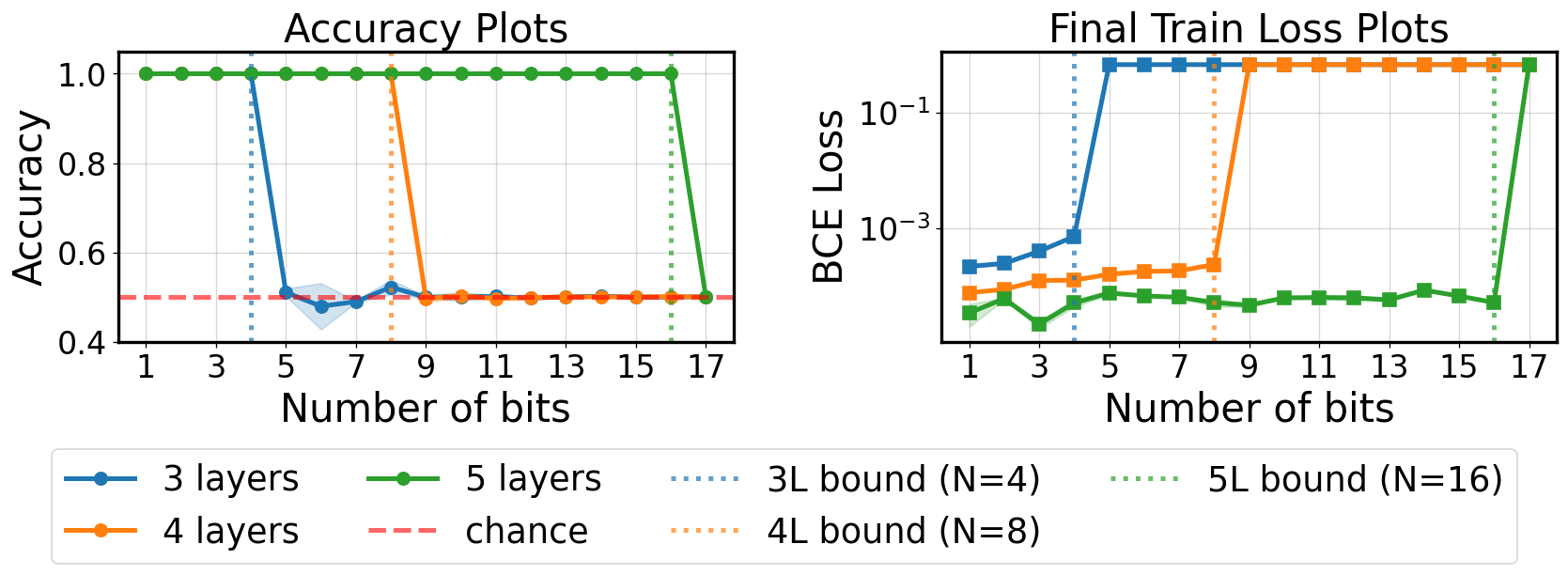}
    \caption{Results of $N$-bit parity benchmarks (up to $17$ bits) applied on \textsc{shonn}s. Specifically, these higher-order models were not equipped with neuron-wise nonlinearities, and also the spectral triadic eigenvectors were not trained, keeping the parameter scaling limited to $O(N^2)$ (further implementation details in Appendix \ref{app:details-on-parity-exps}). We can see that the expressivity of the models is strictly bound by their layer number: a $3$-layer model, for example, can express only up to order $4$ polynomial mappings, regardless of the number of hidden neurons; as a consequence, the optimizer is generally able to push the model right to the bounds of its expressivity, giving rise to the observed sharp drop-off phenomena, as depicted in figure.}
    \label{fig:shonn}
\end{figure*}

In Figure \ref{fig:shonn} the same parity benchmarks are performed on \textsc{shonn}s models (without training the triadic eigenvectors). The comparison with Figure \ref{fig:mlp} immediately reveals crucial differences: the spectral higher-order models achieve perfect accuracies inside their expressivity bounds, to then sharply drop to the random guessing regime once the number of bits in the input goes over their polynomial expressivity.

\section{Conclusions}
\label{sec:conclusions}

Given the parity experiments reported in the last section it appears that:

\begin{itemize}
\item The triadic-eigenvalues-only training paradigm seems to not hinder the expressivity of the higher-order models, if enough hidden neurons are present.
\item Since the expressivity bounds of \textsc{shonn}s clearly scale with their layer number, we are free to increase the number of hidden neurons without polluting the hypothesis space; this in turn allows the optimizer to push the model to the edge of its expressivity, giving rise to sharp transitions that are not obtainable using standard \textsc{mlp}s.
\item \textsc{shonn}s once again seem to combine ease of training with the ability to set sharp expressivity bounds, making them ideal building blocks to introduce a well-defined implicit regularization in neural models.
\end{itemize}

Further studies should aim to shed light on the possibility of employing these sharp expressivity bounds inside other neural models, expanding on the work already present in the \cite{Peri2026SpectralHN}. In particular, implicit regularization stands out as an ideal practical line of application for this research topic, but triadic-augmented \textsc{pinn}s should also be mentioned as a possible interesting prosecution.

\section*{Code Availability}

The code underpinning the presented results can be found at \href{https://github.com/gianluca-peri/shonns-expressiveness}{https://github.com/gianluca-peri/shonns-expressiveness}

\section*{Acknowledgments}    
The authors would like to acknowledge the support by \textsc{\#nextgenerationeu} (\textsc{ngeu}) funded by the Ministry of University and Research (\textsc{mur}), National Recovery and Resilience Plan (\textsc{nrrp}), project \textsc{mnesys} (PE0000006)—A multiscale integrated approach to the study of the nervous system in health and disease (DN. 1553 11.10.2022).

\clearpage
\onecolumn        
\raggedbottom     
\appendices

\section{The Spectral Parameterization}
\label{app:spectral-parameterization}

\subsection*{Spectral Parametrization for Standard Neural Networks}

Consider a weight matrix $W_{K \times N}$ of the links between layers of size $N$ and $K$. The spectral parameterization approach \cite{giambagli2021machine} leverages the implied adjacency matrix between the $N+K$ neurons of the underlying directed bipartite graph:
\begin{equation}
A = \begin{bmatrix}\mathbb{O}_{N \times N} & \mathbb{O} _{N \times K}\\W_{K \times N} & \mathbb{O}_{K \times K}\end{bmatrix}
\end{equation}
The $A$ matrix operates the forward pass on the global activation vector of the neural graph, usually called $\bm{a}$. Crucially, the forward pass of information mediated by $A$ is independent of its diagonal elements. This allows us to consider an augmented version of $A$, with non-zero diagonal entries, without modifying the activations' dynamics. We can then express the modified adjacency matrix in the form $A=\Phi \Lambda \Phi ^{-1}$, with $\Lambda$ diagonal eigenvalues matrix, and with $\Phi$:
\begin{equation}
    \Phi = \begin{bmatrix}\mathbb{I}_{N \times N} & \mathbb{O} _{N \times K}\\\phi_{K \times N} & \mathbb{I}_{K \times K}\end{bmatrix}
\end{equation}
where $\mathbb{I}$ denotes the identity matrix. At last, given that $\Phi ^{-1} = 2\mathbb{I} - \Phi$, it is easy to show that the weights $w_{ki}$, which populate the sub-diagonal block of $A$, can be written as:
\begin{equation}
    w_{ki} = (\lambda _i^{(in)} - \lambda _k^{(out)}) \phi_{ki},
\end{equation}
where $\lambda _i^{(in)}$ is the $i$-th element of the first portion of the eigenvalue diagonal (the first $N$), $\lambda _k^{(out)}$ is the $k$-th of the last one (the last $K$), and $\phi _{ki}$ is the $k,i$ element of the sub-diagonal block $\phi _{K \times N}$.

This novel parametrization offers the versatility to achieve a multitude of different goals, spanning  pruning techniques, input feature relevance detection, architecture search strategies, parameter reduction paradigms, and the training of dynamical systems \cite{chicchi2021training, buffoni2022spectral, chicchi2024automatic, Peri2025, chicchi2025deterministic}.

\subsection*{Spectral Parametrization for Triadic Higher-Order Networks}

Note that in equation \eqref{eq:direct-space-triadic} the $\tilde{w}_{kij}$ tensor element can be thought as a matrix element, $\tilde{w}_{ki'}$: here $i'$ represents an appropriate index which enumerates the set of unordered pairs $(i,j), \ 0 \leq i \leq j<N$ in lexicographical order:
\[
(0,0) \to 0, \ (0,1) \to 1, \ ..., \ (0,N-1) \to N-1, \ (1,1) \to N, \ ..., \ (N-1,N-1) \to \frac{N(N+1)}{2} - 1.
\]
In close form $i'$ is given by:
\begin{equation}
i' = j-i+\sum _{k=0}^{i-1}(N-k).
\label{eq:lexicographical_map}
\end{equation}
The lexicographical bijection \eqref{eq:lexicographical_map} allows the deployment of the spectral parametrization for triadic interactions: we first use \eqref{eq:lexicographical_map} to map the tridimensional weight tensor $[\tilde{w}_{kij}]$ into the weight matrix $[\tilde{w}_{ki'}]$; we can then simply express $[\tilde{w}_{ki'}]$ via the spectral parametrization. So from \eqref{eq:direct-space-triadic} we get:
\begin{equation}
    y_k = \sum _{i} (\lambda_i^{(in)}-\lambda _k^{(out)})\phi_{ki}x_i + \sum _{i'} (\tilde{\lambda} _{i'}^{(in)}-\tilde{\lambda} _k^{(out)})\tilde{\phi}_{ki'}x_{i'}
\end{equation}
with $x_{i'} = x_ix_j$. Finally we can expand $i'$ back to $2$ dimensions:
\begin{equation}
    y_k = \sum _{i} (\lambda_i^{(in)}-\lambda _k^{(out)})\phi_{ki}x_i + \sum _{0 \leq i \leq j} (\tilde{\lambda} _{ij}^{(in)}-\tilde{\lambda} _k^{(out)})\tilde{\phi}_{kij}x_ix_j,
\end{equation}
to obtain equation \eqref{eq:extended-spectral-param}.

\section{Details on the parity experiments}
\label{app:details-on-parity-exps}

Refer to Table \ref{tab:mlp_hyperparameters} for the specifications regarding the \textsc{mlp} experiment. Table \ref{tab:shonn_hyperparameters} refers instead to the \textsc{shonn} trainings.

\begin{table}[ht]
\centering
\begin{tabular}{lr}
\hline
\textbf{Hyperparameter / Setting} & \textbf{Value} \\
\hline
\multicolumn{2}{c}{\textbf{Model Architecture}} \\
\hline
Network Type & Multi-Layer Perceptron (\textsc{mlp}) \\
Hidden Layers & 1 \\
Activation Function & ReLU \\
Output Dimension & 1 \\
\hline
\multicolumn{2}{c}{\textbf{Dataset \& Evaluation}} \\
\hline
Input sequence lengths ($N$) & $1, 2, \dots, 17$ \\
Dataset Size & $2^N$ (all possible $N$-parity inputs) \\
Samples per $N$ & 3 \\
\hline
\multicolumn{2}{c}{\textbf{Optimization}} \\
\hline
Optimizer & Adam \\
Learning Rate & $10^{-3}$ \\
Loss Function & Binary Cross Entropy with Logits \\
Batch Size & $\min(2^{15}, 2^N)$ \\
Max Epochs & 5000 \\
Early Stopping Patience & $\max(100, 20 \times N)$ epochs (on train loss) \\
Gradient Clipping (Max Norm)& 1.0 \\
\hline
\end{tabular}
\caption{Experiment specifications and hyperparameters for the \textsc{mlp} parity task.}
\label{tab:mlp_hyperparameters}
\end{table}

\begin{table}[ht]
\centering
\begin{tabular}{lr}
\hline
\textbf{Hyperparameter / Setting} & \textbf{Value} \\
\hline
\multicolumn{2}{c}{\textbf{Model Architecture}} \\
\hline
Network Type & Spectral Higher-Order Neural Network (\textsc{shonn}) \\
Activation Function & None \\
Hidden Dimension & 128 \\
Output Dimension & 1 \\
\hline
\multicolumn{2}{c}{\textbf{Dataset \& Evaluation}} \\
\hline
Input sequence lengths ($N$) & $1, 2, \dots, 17$ \\
Dataset Size & $2^N$ (all possible $N$-parity inputs) \\
Samples per $N$ & 3 \\
\hline
\multicolumn{2}{c}{\textbf{Optimization}} \\
\hline
Optimizer & Adam \\
Learning Rate & $10^{-3}$ \\
Loss Function & Binary Cross Entropy with Logits \\
Batch Size & $\min(2^{15}, 2^N)$ \\
Max Epochs & 5000 \\
Early Stopping Patience & $\max(100, 20 \times N)$ epochs (on train loss) \\
Gradient Clipping (Max Norm)& 1.0 \\
\hline
\end{tabular}
\caption{Experiment specifications and hyperparameters for the \textsc{shonn} parity task.}
\label{tab:shonn_hyperparameters}
\end{table}

\section{ReLU activated MLPs on Parity tasks}
\label{app:relu-mlp-on-parity}

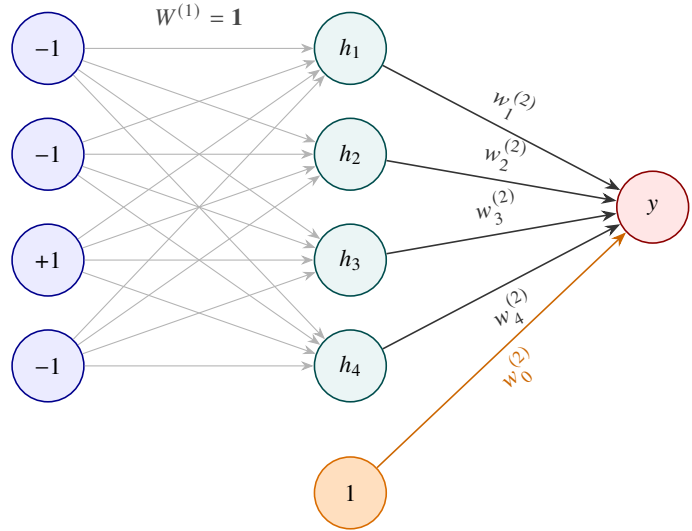
\begin{figure}
    \centering
    \begin{tikzpicture}[
    >=Stealth,
    neuron/.style={
        circle, semithick, minimum size=0.95cm, inner sep=0pt,
        font=\sffamily\small
    },
    input/.style ={neuron, draw=blue!55!black,   fill=blue!8},
    hidden/.style={neuron, draw=teal!60!black,    fill=teal!8},
    bias/.style  ={neuron, draw=orange!85!black,  fill=orange!25},
    output/.style={neuron, draw=red!55!black,     fill=red!10},
    w1link/.style  ={->, thin, black!30},               
    biaslink/.style={->, semithick, orange!80!black},   
    outlink/.style ={->, semithick, black!80},          
]
\def\vs{1.4}   
\def\inX{0}
\def\hidX{4}
\def\outX{8}

\node[input] (x1) at (\inX, { (2.5-1)*\vs }) {$-1$};
\node[input] (x2) at (\inX, { (2.5-2)*\vs }) {$-1$};
\node[input] (x3) at (\inX, { (2.5-3)*\vs }) {$+1$};
\node[input] (x4) at (\inX, { (2.5-4)*\vs }) {$-1$};

\foreach \i in {1,...,4}
    \node[hidden] (h\i) at (\hidX, { (2.5-\i)*\vs }) {$h_\i$};

\node[bias] (b) at (\hidX, { (2.5-5.2)*\vs }) {$1$};

\node[output] (y) at (\outX, 0) {$y$};

\foreach \i in {1,...,4}
  \foreach \j in {1,...,4}
      \draw[w1link] (x\i) -- (h\j);

\draw[biaslink] (b) to
      node[sloped, below, font=\small]{$w_0^{(2)}$}  (y);

\draw[outlink] (h1) -- node[sloped, above, font=\small]{$w_1^{(2)}$} (y);
\draw[outlink] (h2) -- node[sloped, above, font=\small]{$w_2^{(2)}$} (y);
\draw[outlink] (h3) -- node[sloped, above, font=\small]{$w_3^{(2)}$} (y);
\draw[outlink] (h4) -- node[sloped, below, font=\small]{$w_4^{(2)}$} (y);

\node[font=\small, align=center, gray!55!black] at (2, 2.55)
      {$W^{(1)}=\mathbf{1}$};
\end{tikzpicture}
    \caption{Depiction of the kind of neural network instance used in Theorem \ref{Theorem-N}.}
    \label{fig:theo-1}
\end{figure}

In the literature it has been shown that a one-hidden-layer neural network, with $N$ hidden units, is able to perfectly express the $N$-bit parity function \cite{Wilamowski2003SolvingPP}. However, these previous results are restricted to neural networks with threshold-like activation functions, and are not directly applicable to modern, \textit{ReLU} activated, networks, outputting logits. The following theorem covers this modern context.

\begin{theorem}
\label{Theorem-N}
A feedforward neural network with a single, \textit{ReLU} activated, hidden layer of size $N$ is able to compute all the parity mappings up to $N$-bits.
\end{theorem}

\begin{proof}
We shall now focus on proving that a \textsc{mlp} with one hidden layer of size $N$ is able to perfectly solve the $N$-bit parity problem. The fact that it's also able to solve all the other parity tasks up to $N$-bits will then easily follow.

Let the input vector be $\bm{x} \in \{-1, 1\}^N$ and consider a one-hidden-layer, \textit{ReLU} activated, \textsc{mlp} with $N$ hidden neurons. We start by constraining the weights of the first layer $w_{ij}^{(1)}$ to be identically $1$:
\begin{equation}
    w_{ij}^{(1)} = 1 \quad \forall i,j.
\end{equation}
Consequently, the inputs to all neurons in the hidden layer are simply the same sum $S(\bm{x}) = \sum_i x_i$. Because $x_i \in \{-1, 1\}$, the sum $S(\bm{x})$ can only take $N+1$ distinct values:
\begin{equation}
    S_k \in \{-N, -N+2, \dots, N-2, N\} \quad k \in \{0, \dots, N\},
\end{equation}
where $k$ runs on the possible input sequences.

We now set the hidden biases to sit between these possible values of $S(\bm{x})$: $\bm{b}=\{N-1, N-3,...,-N+3,-N+1\}$, \textit{i.e.}
\begin{equation}
    b_j = N - 2j + 1 \quad j \in \{1, \dots, N\}.
\end{equation}
The post non-linearity activation of the $j$-th hidden neuron is then $h_j = \max(0, S(\bm{x}) + b_j)$. 

The network's final output is a linear combination of the $h$s, followed by a sigmoid activation $\sigma(\cdot)$. Let $z$ be the pre-sigmoid logit:
\begin{equation}
    z = w^{(2)}_0 + \sum_{j=1}^N w^{(2)}_j h_j,
\end{equation}
note that we are adopting the convention of the $0$-th \textit{bias neuron} (backward-detached neuron with constant unitary activation) to express the bias term of the last layer (see Figure \ref{fig:theo-1} for a graphical depiction). We shall have a different value of $\bm{h}$ for each possible input combination (enumerated by $k$). Note now that to compute the parity function, and minimize the BCE loss, it is sufficient for the logit $z$ to be able to alternate between $-M,M, \ M>>0$ as $k$ increases (e.g., $M$ even parity, $-M$ for odd parity).

The trick now is to consider the square matrix $A_{(N+1) \times (N+1)}$, whose rows span the possible $k$ values present in the hidden layers' neurons after  the non-linearity, including the bias neuron: specifically the rows of $A$ will contain the distinct $N+1$ activity vectors of the hidden layer $\bm{h}(k)$. It is easy to convince oneself that the resulting matrix will be a particular kind of bordered Toeplitz matrix:

\begin{equation}
A=
\begin{pmatrix}
0 & 0 & 0 & \cdots & 0  & 0 & 1\\
1 & 0 & 0 & \cdots & 0  & 0 & 1\\
2 & 1 & 0 & \cdots & 0  & 0 & 1\\
3 & 2 & 1 & \cdots & 0  & 0 & 1\\
\vdots & \vdots & \vdots & \ddots & \vdots & \vdots & \vdots\\
N-1  & N-2  & N-3  & \cdots & 1 & 0 & 1\\
N    & N-1  & N-2  & \cdots & 2 & 1 & 1
\end{pmatrix}_{(N+1) \times (N+1)}
\end{equation}
that is promptly proven invertible via Laplace expansion on the first row, giving determinant equal to $(-1)^N$.

As we already stated, to compute the parity function, and minimize the BCE loss, it is sufficient for the logit $z$ to be able to alternate between $-M,M$ as $k$ increases. This concretely amounts to finding a single weight vector $\bm{w}^{(2)}$ that solves the linear system:
\begin{equation}
    A\bm{w}^{(2)} = \bm{t}
\end{equation}
with $\bm{t}$ having elements $t_i=(-1)^iM$. But this is surely possible given that $A$ is invertible.

We have thus proven that a \textsc{mlp} with $N$ hidden neurons in one layer is able to perfectly express the $N$-bit parity mapping; but since we can always detach hidden neurons by setting all their incoming and outgoing weights to zero, it also immediately follows that the same network is able to express all the $M$-bit parity tasks with $M \leq N$.
\end{proof}

\newpage
\bibliographystyle{plain}
\bibliography{bibliography}

\end{document}